\documentclass[runningheads]{llncs}

\usepackage{graphicx}
\usepackage{amsmath}

\usepackage{amsfonts}
\usepackage{nicefrac}
\usepackage[section]{placeins}
\usepackage{rotating}
\usepackage{fancybox}
\usepackage{enumerate}
\usepackage{algorithm}
\usepackage{algorithmic}
\usepackage{amssymb}
\newcommand{\bu}{\boldsymbol{u}}
\newcommand{\bU}{\boldsymbol{U}}

\newcommand{\bl}{\boldsymbol{l}}

\newcommand{\bW}{\boldsymbol{W}}

\newcommand{\bS}{\boldsymbol{S}}
\newcommand{\bL}{\boldsymbol{L}}

\newcommand{\bx}{\boldsymbol{x}}

\newcommand{\be}{\boldsymbol{e}}

\newcommand{\bX}{\boldsymbol{X}}

\newcommand{\bgamma}{\boldsymbol{\gamma}}
\newcommand{\brho}{\boldsymbol{\rho}}
\newcommand{\btheta}{\boldsymbol{\theta}}
\newcommand{\bdelta}{\boldsymbol{\delta}}

\newcommand{\TP}{\mathrm{TP}}

\newcommand{\FP}{\mathrm{FP}}
\newcommand{\FN}{\mathrm{FN}}

\DeclareMathOperator{\arctg}{arctg}

\DeclareMathOperator{\argmax}{argmax}

\usepackage{xcolor}

\begin{document}


\title{Multitask Hopfield Networks}

\author{Marco Frasca\inst{1}\orcidID{0000-0002-4170-0922}, Giuliano Grossi\inst{1}\orcidID{0000-0001-9274-4047} and Giorgio Valentini\inst{1}\orcidID{0000-0002-5694-3919}}

\institute{Dipartimento di Informatica, Universit\`a degli Studi di Milano, Via Celoria 18, 20135 Milano, Italy\\
\email{\{frasca,grossi,valentini\}@di.unimi.it}
}

\maketitle
\begin{abstract}
Multitask algorithms typically use task similarity information as a bias to speed up 
and improve the performance of learning processes. 
Tasks are learned jointly, sharing information across them, in
order to construct models more accurate than those learned separately
over single tasks. In this contribution, we present the first multitask
model, to our knowledge, based on Hopfield Networks (HNs), named
HoMTask. We show that by appropriately building a unique HN embedding all tasks, 
a more robust and effective classification model can be learned. 
HoMTask is a transductive semi-supervised parametric HN, that
minimizes an energy function extended to all nodes and to all tasks under
study. We provide theoretical evidence that the optimal parameters automatically 
estimated by HoMTask make coherent the model itself with the prior knowledge 
(connection weights and node labels). The convergence properties of HNs are preserved, 
and the fixed point reached by the network dynamics gives rise to the prediction of unlabeled nodes. The proposed model improves the classification abilities of singletask HNs on a preliminary benchmark comparison, and achieves
  competitive performance with  state-of-the-art semi-supervised graph-based algorithms.



\end{abstract}




\section{Introduction}
Multitask learning is concerned with simultaneously learning multiple prediction tasks that are related to
one another. It has been frequently observed in the recent literature that, when there are relations between the tasks, it can be advantageous to learn them simultaneously instead of learning each task separately~\cite{Caruana97,Evgeniou05}. A major challenge in multitask learning is how to selectively screen the sharing of information so that unrelated tasks do not end up influencing each other. Sharing information between two unrelated tasks can worsen the performance of both tasks.

Multitasking thus plays an important role in a variety of practical situations, including: the prediction of user ratings for unseen items based on rating information from related users~\cite{Ning10}, the simultaneously forecasting of many related financial indicators~\cite{Greene02}, the categorization of genes associated with a genetic disorder by exploiting genes associated with related diseases~\cite{gene2disco}.

There is a vast literature on multitask learning. The most important lines of work include: regularizers biasing the solution towards functions that lie geometrically close to each other in a RKHS~\cite{Evgeniou04,Evgeniou05}, or lie in a low dimensional subspace~\cite{argyriou2008convex,kang2011learning}; structural risk minimization methods, where multitask relations are established by enforcing predictive functions for the different tasks to belong to the same hypothesis set~\cite{Ando05}; spectral~\cite{evgeniou2007multi,argyriou2007spectral} and cluster-based~\cite{jacob2009clustered,xue2007multi} assumptions on the task relatedness; Bayesian approaches where task parameters share a common prior~\cite{yu2005learning,daume2009bayesian,zhou2011clustered}; methods allowing a small number of outlier tasks that are not related to any other task~\cite{yu2007robust,chen2011integrating}; approaches attempting to learn the full task covariance matrix~\cite{zhang2010learning,guo2011sparse}.
To our knowledge, no multitask attempts have been proposed for Hopfield networks (HNs)~\cite{Hop82}, whereas several studies investigated HNs as singletask classifier~\cite{Jacyna89,Karaoz04,Sion13,Hu17}. Indeed, HNs are efficient local optimizers, using the local minima of the energy function determined by network dynamics as a proxy to node classification. 

In this paper we develop \textit{HoMTask}, \textit{Hopfield multitask Network}, an approach to multitask learning based on exploiting a family of parametric HNs. Our approach builds on COSNet~\cite{Bertoni11}, a singletask HN  proposed to classify instances in a transductive semi-supervised scenario with unbalanced data. 
A main feature of HoMTask is that the energy function is extended to all tasks to be learned and to all instances (labeled and unlabeled), so as to learn the model parameters and to infer the node labels  simultaneously for all tasks. The obtained network can be seen as a collection of singletask HNs, appropriately interconnected by exploiting the task relatedness.  In particular, each task is associated with a couple of parameters determining the neuron activation values and thresholds, and we theoretically proved that in the optimal case, the learning procedure adopted is able to learn the parameters so as to move the multitask state of the labeled sub-network to a minimum of the energy.  
This is an important result, which allows the model to better fit the input data, since the classification of unlabeled nodes is based upon a minimum of the unlabeled subnetwork. 
Another interesting feature of HoMTask is that the complexity of the learning procedure linearly increases with the number of tasks, thus allowing the model to nicely scale on settings including numerous tasks.
Finally, a proof of convergence of the multitask dynamics to a minimum of the energy is also supplied. 

Experiments on a real-world classification problem have shown that HoMTask remarkably outperforms singletask HNs, and has competitive performance with state-of-the-art graph-based methods proposed in the same context.
\section{Methods}

\subsection{Problem definition}\label{sub:problem}
The problem input is composed of
 an undirected weighted graph $G(V, \bW)$, where $V=\{1, 2, \ldots, n\}$ is the set of instances and the non negative symmetric matrix $\bW=\left(w_{ij}\right)$ denotes the degree of functional similarity between each pair of nodes $i$ and $j$. 
A set of binary learning tasks $C=\{c_k| k=1, 2, \dots, m\}$ over $G$ is given, 
where for every task $c_k$, $V$ is labelled with $\lbrace +, - \rbrace$. The labeling is known only for the subset $L\subset V$, whereas it is unknown for $U:=V\setminus L$. Moreover, the subsets of vertices labelled with $+$ (positive) and $-$ (negative) are denoted by $L_{k,+}$ and $L_{k,-}$, respectively, for each task $c_k \in C$.
Without loss of generality, we assume $U=\{1, 2, \cdots, h\}$ and $L = \{h+1, h+2, \cdots, n\}$. As further assumption, task labelings are highly unbalanced, that is $\frac{|L_{k,+}|}{|L_{k,-}|}\ll1$, for each $k\in \{1, 2, \ldots, m\}$.
In the multitask scenario, a $m \times m$ symmetric matrix $\bS=s_{kr}|_{k,r=1}^m$ is also given, where $s_{kr} \in [0,1]$ is an index of relatedness/similarity between the tasks $c_k$ and $c_r$, and $s_{kk}=0$ for each $k \in \{1, 2, \dots, m\}$. 

The aim is determining a set of bipartitions $(U_{k,+},U_{k,-})$ of vertices in $U$ for each task $c_k \in C$ by jointly learning tasks in $C$, on the basis of the prior information encoded in $G$ and $\bS$.

In the following, the bold font is adopted to denote vectors and matrices, and the calligraphic font to denote multitask Hopfield networks. {Moreover, we denote by $\bW_{LL}$ and $\bW_{UU}$ the submatrices of $\bW$ relative to nodes $L$ and $U$, respectively}. 
\subsection{Previous singletask modelling}
In this section we recall the basic model proposed in \cite{Bertoni11,Frasca13} for singletask modeling, named \textit{COSNet}, that has inspired the multitask setting presented here.
Essentially, it relies on a parametric family of the Hopfield model~\cite{Hop82}, where  the network parameters are learned to cope with the label imbalance and the network equilibrium point is interpreted to classify the unlabeled nodes. A COSNet network over $G=\langle V, \bW\rangle$ is a the triple $H=\langle \bW, \lambda, \rho\rangle$, where
 $\lambda\in \mathbb{R}$ denotes the neuron activation threshold (unique for all neurons), and $\rho\in[0,\frac{\pi}{2})$ is a parameter which determines the two neuron activation (state) values $\{\sin\rho,-\cos\rho\}$. The model parameters are appropriately learned in order to allow the algorithm to counterbalance the large imbalance towards negatives (see~\cite{Frasca13}).
The initial state of a neuron $i\in V$ is set to $x_i(0) = \sin\rho$, if $i$ is positive,  $x_i(0) = -\cos\rho$, if $i$ is negative, and $x_i(0) = 0$ when  $i$ in unlabeled. The network evolves according to the following asynchronous dynamics:
\begin{equation}
\begin{small}
{\begin{scriptsize} x_i(t)= \end{scriptsize}} \left\{ \begin{array}{rl}
\sin\rho &\mbox{\hspace{0.3cm}if\ } \overset{i-1}{\underset{j=1}\sum} w_{ij}x_j(t) + \overset{n}{\underset{k=i+1}\sum} w_{ik}x_k(t-1) - \lambda > 0\\
-\cos\rho &\mbox{\hspace{0.3cm}if\ } \overset{i-1}{\underset{j=1}\sum} w_{ij}x_j(t) + \overset{n}{\underset{k=i+1}\sum} w_{ik}x_k(t-1) - \lambda \leq 0\\
\end{array}
\right.
\label{eq:HN_update}
\end{small}
\end{equation}
where $x_i(t)$ is the state of neuron $i \in V$ at time $t$. 
At each time $t$, the vector $\bx(t) = (x_1(t), x_2(t), \ldots, x_h(t))$ represents the state of the whole network. The network admits a state function named \textit{energy function}:
\begin{equation}
E(\bx) = -\frac{1}{2} \sum_{i \neq j}{ w_{ij}x_i x_j } + \lambda\sum_{i = 1}^{n}{x_i}.
\label{eq:energy}
\end{equation} 
The convergence properties of the dynamics~(\ref{eq:HN_update}) depend on the weight matrix structure $\bW$ and the rule by which the nodes are updated. In particular, if the matrix is symmetric and the dynamic is asynchronous, it has been proved that the network converges to a stable state in polynomial time.
As a major result, it has been shown that \eqref{eq:energy} is a Lyapunov function for the Hopfield dynamical systems with  asynchronous dynamics, i.e., for each $t>0$, $E(\bx(t+1))\leq E(\bx(t))$ and exists a time $\bar t$ such that $E(\bx(t))=E(\bx(\bar t))$, for all $t\geq\bar t$.    
Moreover, the reached fixed point $\bar\bx=\bx(\bar t)$ is a local minimum of \eqref{eq:energy}. Then, a neuron $i$ in $U$ is classified as positive if $\bar x_i = \sin\rho$, as negative otherwise. 


\subsection{Multitask Hopfield networks}\label{sub_homtask}
A \textit{Hopfield multitask network}, named \textit{HoMTask},  with neurons $V$ is a quadruple $\mathcal{H}=\langle \bW, \bgamma, \brho, \bS\rangle$, where $\bS$ is the task similarity matrix, $\bgamma = (\gamma_1, \ldots, \gamma_m) \in \mathbb{R}^m$, $\brho = (\rho_1, \ldots, \rho_m) \in [\frac{\pi}{4}, \frac{\pi}{2})^m$. {The couple of parameters $(\gamma_k,\rho_k)$ is associated with task $c_k$,  for each $k \in \{1,2,\ldots,m\}$: by leveraging the approach adopted in COSNet, for a task $c_k$, the neuron activation values are $\{\sin\rho_k, -\cos\rho_k \}$, whereas $\gamma_k$ is the neuron activation threshold (the same for every neuron). Such a formalization allows to keep the absolute activation value in the range $[0,1]$, and to calibrate it by suitably learning $\rho_k \in [\frac{\pi}{4}, \frac{\pi}{2})$. For instance, in presence of a large majority of negative neurons, $\rho_k$ close to $\frac{\pi}{2}$ would prevent positive neurons to be  overwhelmed during the net dynamics.}

The state of the network is the $n\times m$ matrix $\bX = (\bx^{(1)}, \bx^{(2)}, \ldots, \bx^{(m)})$, where $\bx^{(k)} = (x_{1k},\ x_{2k},\ \ldots, x_{nk})\in \{\sin\rho_k, -\cos\rho_k \}^n$ is the state vector corresponding to task $c_k$. 
When simultaneously learning related tasks $c_k$ and $c_r$, an usual approach consists in   expecting that the higher the relatedness $s_{rk}$, the closer the corresponding states. In our setting, this can be achieved by minimizing  
\begin{displaymath}
\label{eq:norm_diff}
 \Vert \bx^{(k)} - \bx^{(r)} \Vert^2\ ,
\end{displaymath}
for any couple of tasks $c_k, c_r \in C$, with $k\neq r$. {To this end, we incorporate  a term proportional to $\sum_k\sum_r s_{kr} \Vert \bx^{(k)} - \bx^{(r)} \Vert^2$ into the energy of $\mathcal{H}$, thus obtaining:}
\begin{equation}\label{eq:ExtEnergy}
E_\mathcal{H}(\bX) = \sum_{k=1}^m \left( E\big(\bx^{(k)}\big) + \frac{\alpha}{4} \sum_{\substack{r=1\\r\neq k}}^{m} s_{kr}\ \Vert \bx^{(k)} - \bx^{(r)} \Vert^2 \right),
\end{equation}
where $E\big(\bx^{(k)}\big) = -\frac{1}{2}{\bx^{(k)}}^T\bW\bx^{(k)} + {\bx^{(k)}}^T\gamma_k\be_n$, $\be_n$ is the $n$-dimensional vector made by all ones, and $\alpha$ is a
real hyper-parameter regulating the multitask contribution. 
Without the second additive term in brackets, energy~(\ref{eq:ExtEnergy}) would be the summation of the energy functions of $m$ independent singletask Hopfield networks, as recalled in the previous section.   

By using the equalities 
\begin{displaymath}
\Vert \bx^{(k)} - \bx^{(r)} \Vert^2 = \Vert \bx^{(k)}\Vert^2 + \Vert \bx^{(r)}\Vert^2 -2 \bx^{(k)}\cdot \bx^{(r)}\ ,
\end{displaymath}
 where $\cdot$ denotes the inner product, and giving that
  \begin{displaymath}
\sum_{k=1}^m\sum_{\substack{r=1\\r\neq k}}^{m}s_{kr} \big(\Vert \bx^{(k)}\Vert^2 + \Vert \bx^{(r)}\Vert^2\big) = 2\Big(\sum_{\substack{k=1}}^m S_k\Vert \bx^{(k)}\Vert^2 \Big)\ ,
 \end{displaymath}
with $S_{k}= \sum_{r=1}^{m}s_{kr}$, the energy (\ref{eq:ExtEnergy}) can be rewritten as:
\begin{equation}\label{eq:ExtEnergyS}
E_\mathcal{H}(\bX) = \sum_{k=1}^m \left( E\big(\bx^{(k)}\big) + \frac{\alpha}{2} \Big(S_k\sum_{i=1}^n x_{ik}^2 - \sum_{\substack{r=1\\r\neq k}}^{m} s_{kr}\sum_{i=1}^n x_{ik}x_{ir}\Big)
\right) .
\end{equation}

Informally,  $\mathcal{H}$ can be thought as $m$ interconnected singletask parametric Hopfield networks {$H_1=\langle \bW, \bgamma_1, \brho_1\rangle, \ldots, H_m=\langle \bW, \bgamma_m, \brho_m\rangle$} on $V$, having all the same topology given by $\bW$. 
In addition, the multitask energy term 
introduces self loops for all neurons, and a novel connection for each neuron $i\in V$ with $i$ itself in the network $H_r$, $r\in C\setminus c_k$, whose weight is $\alpha s_{kr}$ (see Fig.\ref{fig:Hsim}).
\begin{figure}[t]
\begin{center}
\caption{Topology of $\mathcal{H}$ in the case $m=3$. Black circles, gray squares and white circles represent elements of $L_-$, $L_+$ and $U$ respectively. The local topology is the same across sub-networks $\mathcal{H}_1$, $\mathcal{H}_2$ and $\mathcal{H}_3$, but the labeling varies with the task.}\label{fig:Hsim}
\vspace{0.15cm}
\scriptsize
\hspace{-0.3cm}\includegraphics[width=0.8\textwidth]{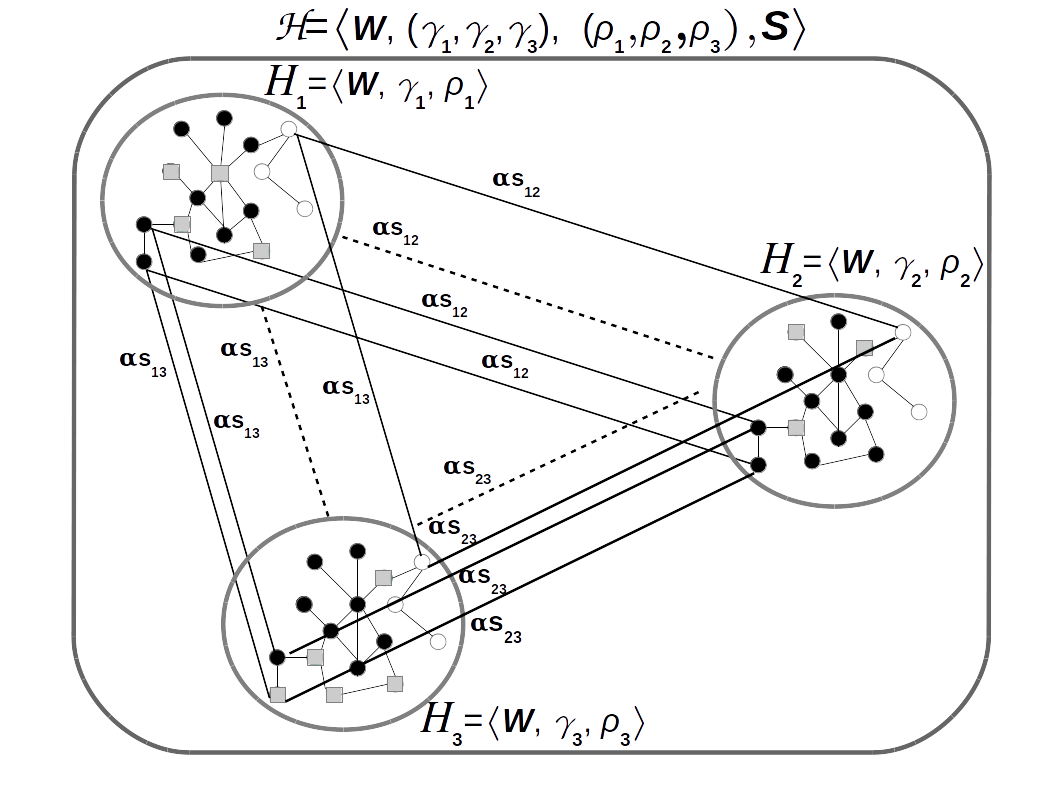}
\normalsize
\vspace{-1cm}
\end{center}
\end{figure}
It is worth nothing that usually in Hopfield networks there are no self-loops; nevertheless, we show that it does not affect the convergence properties of the overall network.  
\subsubsection{Update rule and dynamics convergence.}
Starting from an initial state $\bX(0)$ and adopting the asynchronous dynamic, in $nm$ steps all neurons are updated in random order according to the following update rule:
\begin{equation}\label{eq:HN_X_update}
x_{ik}(t+1)=\begin{cases}
\sin\rho_k, & \mbox{if\ } \phi_{ik}(t) > 0 \\
-\cos\rho_k, &\mbox{if\ } \phi_{ik}(t) \leq 0
\end{cases}
\end{equation}
where $x_{ik}(t+1)$ is the state of neuron $i\in X$ in task $c_k$ ($ik$-th) at time $t+1$,  
and \small
\begin{equation}\label{eq:phi_ik}
 \phi_{ik}(t) := A_{ik}(t) - \theta_{ik} + \alpha B_{ik}(t)
\end{equation} 
is the input of the {$ik$}-th neuron at time $t$,  whose terms are $A_{ik}(t)= \sum\limits_{j=1}^{n}w_{ij}x_{jk}(t)$, $\theta_{ik}=\gamma_k+ \frac{\alpha S_k}{2}\big(\sin\rho_k-\cos\rho_k\big)$, and $ B_{ik}(t)= \sum\limits_{\substack{r=1\\r\neq k}}^{m}s_{kr}x_{ir}(t)$.
{$A_{ik}$ represents the singletask input (eq.~\ref{eq:HN_update}), $B_{ik}(t)$ is the multitask contribution, and $\theta_{ik}$ is the activation threshold for neuron $ik$, including also the `singletask' threshold. The form of $\theta_{ik}$ derives from the following theorem, stating 
a HoMTask Hopfield network preserves the convergence properties of a Hopfield network.

\begin{theorem}{} 
A \textit{HoMTask Hopfield network} $\mathcal{H}=\langle \bW, \bgamma, \brho, \bS\rangle$ with $n$ neurons and the asynchronous dynamics (\ref{eq:HN_X_update}), which starts from any given network state, eventually reaches a stable state at a local minimum of the energy function~(\ref{eq:ExtEnergyS}).
\label{th:converg}\end{theorem} 
\begin{proof} Let $E_{ik}(t)$ be the energy contribution to (\ref{eq:ExtEnergyS}) of the $ik$-th neuron at time $t$, with
 \begin{equation*}
 \begin{split}
E_{ik}(t)=-\frac{1}{2}x_{ik}(t)\sum\limits_{j=1}^h (w_{ij}+w_{ji})x_{jk}(t)+\gamma_{k}x_{ik} +&\frac{\alpha S_k}{2} x_{ik}^2 -\\& \frac{\alpha}{2} x_{ik}\sum\limits_{\substack{r=1\\r\neq
 k}}^m (s_{kr}+s_{rk}) x_{ir}\ .
\end{split}
\end{equation*}

Let $\Delta_{ik}E(t+1) = E_{ik}(t+1)-E_{ik}(t)$ be the energy variation after updating the state $x_{ik}$ at time $t+1$ according to (\ref{eq:HN_X_update}). Due to the symmetry of $\bW$ and $\bS$, it follows
\begin{equation}
\begin{split}
\textstyle \Delta_{ik}E(t+1)=  -\big(x_{ik}&(t+1)-x_{ik}(t)\big)\\&\Big(A_{ik}(t)-\gamma_{k} - \frac{\alpha S_k}{2}\big(x_{ik}(t+1)+x_{ik}(t)\big)+\alpha B_{ik}(t)\Big).
\end{split}
\label{eq:delta_ik}
\end{equation}
\normalsize 
Since the energy (\ref{eq:ExtEnergyS}) is lower bounded, to complete to proof we need to prove that after updating $x_{ik}$ at time $t+1$ according to (\ref{eq:HN_X_update}), it holds $\Delta_{ik}E(t+1) \leq 0$. From (\ref{eq:delta_ik}), when $x_{ik}(t+1) = x_{ik}(t)$, that is when the neuron does not change state, it follows $\Delta_{ik}E(t+1) = 0$. Accordingly, we need to investigate the remaining two cases: (a) $x_{ik}(t)=\sin\rho_k$ and $x_{ik}(t+1)=-\cos\rho_k$; (b) $x_{ik}(t)=-\cos\rho_k$ and $x_{ik}(t+1)=\sin\rho_k$. In both cases it holds (by definition of $\theta_{ik}$) $\gamma_{k} - \frac{\alpha S_k}{2}\big(x_{ik}(t+1)+x_{ik}(t)\big)\ \substack{=}\ \theta_{ik}$.
\begin{enumerate}[(a)]
\item  $(x_{ik}(t+1)-x_{ik}(t)) = (-\cos\rho_k -\sin\rho_k)< 0$, and, according to~(\ref{eq:HN_X_update}), $A_{ik}(t) - \theta_{ik} + \alpha B_{ik}(t) \leq 0$. It follows $\Delta_{ik}E(t+1)\leq 0$.
\item $(x_{ik}(t+1)-x_{ik}(t)) = (\sin\rho_k +\cos\rho_k)> 0$, and $A_{ik}(t) - \theta_{ik}  + \alpha B_{ik}(t) > 0$. Thus $\Delta_{ik}E(t+1) < 0$.
\end{enumerate} 
Every neuron update thereby does not increase the energy of the network, and, since the energy is lower bounded, there will be a time $t'>0$ from which the update of any given neuron will not change the current state, which is the definition of equilibrium state of the network, and which makes $\bX(t')$ a local minimum of (\ref{eq:ExtEnergyS}).
\qed
\end{proof}
\subsubsection{Learning the model parameters.}
Considered the subnetwork $\mathcal{H_L}=\langle \bW_{\textrm{LL}}, \bgamma, \brho, \bS\rangle$ restricted to labeled nodes $L$, its energy is:
 \begin{equation}\label{eq:ExtEnergyS_L}
E_\mathcal{H_L}(\bL) = \sum_{k=1}^m \left( E_L\big(\bl^{(k)}\big) + \frac{\alpha}{2} \Big(S_k\sum_{i\in L} l_{ik}^2 - \sum_{\substack{r=1\\r\neq k}}^{m} s_{kr}\sum_{i\in L} l_{ik}l_{ir}\Big)
\right)\ ,
\end{equation}
where $\bL=(\bl^{(1)}, \bl^{(2)}, \ldots, \bl^{(m)})$ with components $\bl^{(k)}=(l_{1k},l_{2k},\dots, l_{(n-h)k})$ belonging to the set $\{\sin\rho_k, -\cos\rho_k \}^{(n-h)}$, and $E_L\big(\bl^{(k)}\big) = -\frac{1}{2}{\bl^{(k)}}^T\bW_{\textrm{LL}}\bl^{(k)} + {\bl^{(k)}}\cdot\gamma_k\be_{(n-h)}$.

The given bipartition $(L_{k,+},L_{k,-})$ for each task $c_k$ naturally induces the labeling $\bar\bl^{(k)}=\{\bar l_{1k}, \bar l_{2k}, \ldots,\bar l_{(n-h)k}\}$, defined as it follows:
\begin{equation*}
\bar l_{ik}=\begin{cases}
\sin\rho_k, &\mbox{if\ }  i \in L_{k,+}\\
-\cos\rho_k, &\mbox{if\ } i \in L_{k,-}\\
\end{cases},
\end{equation*}
and constituting the known {`multitask'} state $\bar\bL=(\bar\bl^{(1)}, \bar\bl^{(2)}, \ldots, \bar\bl^{(m)})$. 


Given $\bar\bL$ as known components of a final state $\bar\bX$ of the multitask network $\mathcal{H}=\langle \bW, \bgamma,\brho, \bS\rangle$, the purpose of the learning step is to compute the pair ($\hat\bgamma$, $\hat\brho$) which makes $\bar\bX$ an energy global minimizer of \eqref{eq:ExtEnergy}, the energy function associated with $\mathcal{H}$. Since our aim is also keeping the model scalable on large sized data, and finding
the global minimum of the energy requires time/memory intensive procedures, we employ a learning procedure leading $\bar\bL$ towards an fixed point of $\mathcal{H}_L$, being in general a local minimum of \eqref{eq:ExtEnergyS_L}. 
We provide the details of the learning procedure in the following, showing that such an approach also helps to handle the label imbalance at each task. 
\paragraph{Maximizing a cost-sensitive criterion.}
When the parameters $\bgamma, \brho$ are fixed, each neuron $ik$ has input 
\[
\begin{split}
    \phi^L_{ik}(\bgamma, \brho) = \sum_{j\in L}w_{ij}\Big(\sin\rho_k\chi_{jk} -\cos\rho_k&\big(1-\chi_{jk}\big)\Big) - \beta_k +\\& \alpha\sum_{\substack{r=1\\r\neq k}}s_{kr}\Big(\sin\rho_k\chi_{ir} -\cos\rho_k\big(1-\chi_{ir}\big)\Big)\ ,
    \end{split}
\]
where, for each $k\in \{1,\dots, m\}$ and $j\in L$, $\chi_{jk}=1$ if $j \in L_{k,+}$, $0$ otherwise. 
$\phi^L_{ik}$ corresponds to $\phi_{ik}$ of equation~(\ref{eq:phi_ik}) restricted to $L$; to simplify the notation,  in the following $\phi^L_{ik}$ is thereby denoted by $\phi_{ik}$. Since the subnetwork is labeled, it is possible to define the set of \textit{true positive} $tp_k(\bgamma, \brho)=\{i \in L_{k,+}|\phi_{ik}(\bgamma, \brho)>0\}$, \textit{false negative} $fn_k(\bgamma, \brho)=\{i \in L_{k,+}|\phi_{ik}(\bgamma, \brho)\leq 0\}$, and \textit{false positive} $fp_k(\bgamma, \brho)=\{i \in L_{k,-}|\phi_{ik}(\bgamma, \brho)>0\}$, for every task $c_k$. 
Following the approach proposed in~\cite{homcat}, a set of membership functions can be defined, extending the crisp memberships introduced above:

\begin{equation}\label{eq:fuzzy_memb}
\begin{array}{l}
\TP(i,k,\bgamma, \brho) = f(\tau\phi_{ik}(\bgamma, \brho)), \hspace{3.5cm} i \in L_{k,+}\\ [.1cm]
\FN(i,k,\bgamma, \brho) = 1 - f(\tau\phi_{ik}(\bgamma, \brho)), \hspace{2.9cm} i \in L_{k,+}\\ [.1cm]
\FP(i,k,\bgamma, \brho) = f(\tau\phi_{ik}(\bgamma, \brho)),\hspace{3.45cm} i \in L_{k,-}
\end{array}
\end{equation}
where $f:\mathbb{R} \rightarrow [0,1]$ is a suitable monotonically increasing membership function. For instance $f_1(x) = \nicefrac{1}{\big(1+e^ {(-x)}\big)}$ or $f_2(x) = \frac{1}{2}\big(\frac{2}{\pi} \arctg(x)+1\big)$. $\tau >0$ is a real parameter. If $f$ is the Heaviside step function, we obtain the crisp memberships. 
For example, when  $f=f_1$ or $f=f_2$, if $i \in L_{k,+}$ and $\tau \phi_{ik}(\bgamma, \brho) = 0$, if follows $\TP(i, k, \bgamma, \brho) =\FN(i, k, \bgamma, \brho) = 0.5$; if $i \in L_{k,+}$ and $\tau \phi_{ik}(\bgamma, \brho) \rightarrow \infty$, it follows $\TP(i, k, \bgamma, \brho) = 1$ and $\FN(i, k, \bgamma, \brho) = 0$. 
The intermediate cases lead to $ 0 <\TP(i,k, \bgamma, \brho), \FN(i,k, \bgamma, \brho) < 1$.

Such a generalization, in a different setting (singletask, multi-category) increased both the learning capability of the model and its classification performance~\cite{homcat}. 
By means of the membership functions~(\ref{eq:fuzzy_memb}), we can define the objective $F$:

\begin{equation}\label{eq:F}
\textstyle
\mbox{\textit{F}}(\bgamma, \brho) =  \sigma\big(F_1(\bgamma, \brho), F_2(\bgamma, \brho), \ldots, F_m(\bgamma, \brho)\big)\ ,
\end{equation}
\normalsize
where $$F_k(\bgamma, \brho)=  \frac{2\sum\limits_{i \in L_{k,+}}{\TP(i,k, \bgamma, \brho)}}{2\sum\limits_{i \in L_{k,+}}{\TP(i,k, \bgamma, \brho)} + \sum\limits_{i \in L_{k,-}}{\FP(i,k, \bgamma, \brho)} + \sum\limits_{i \in L_{k,+}}{\FN(i, k, \bgamma, \brho)}}$$
and $\sigma$ is an appropriately chosen function, e.g. the mean, the minimum, or the harmonic mean function. The property $\sigma$ must satisfy is that 

\[F(\bgamma, \brho) = 1 \Longrightarrow F_1(\bgamma, \brho)=F_2(\bgamma, \brho)=\ldots =F_m(\bgamma, \brho)=1.
\]
By definition, $F_k$ (a generalization of the F-measure) is penalized more by the misclassification of a positive instance than by the misclassification of a negative one. By maximizing $F(\bgamma, \brho)$ we can thereby cope with the label imbalance. To this end, the learning criterion for the model parameters adopted here is $
\displaystyle(\hat\bgamma, \hat\rho) = \arg\underset{\bgamma,\brho} {\max} \hspace{0.1cm} \mbox{\textit{F}}(\bgamma, \brho)$, which also leads to the following important result.
\begin{theorem}
If $\mbox{\textit{F}}(\bgamma,\brho) = 1$, then $\bar\bL$ is an equilibrium state of the sub-network $\mathcal{H_L}\langle W_{\textrm{LL}}, \bgamma, \brho, \bS\rangle$.
\label{fact:F-equilibrium}
\end{theorem}

\paragraph{Learning procedure.}
Denoted by $\bdelta = (\bgamma,\brho)$ the vector of model parameters, this procedure learns the values $\hat\bdelta$ that maximize eq. (\ref{eq:F}), that is $\displaystyle\hat\bdelta =  {\arg}\underset{\bdelta}{\max} \hspace{0.1cm} \mbox{$F$}(\bdelta).$  
Following the approach proposed in~\cite{homcat}, we adopt the \textit{simplest search  method}~\cite{Kordos08}, which adopts an iterative and incremental procedure estimating in turn a single parameter at a time,  by fixing the other ones, until a suitable  criterion is met (e.g. convergence, number of iterations, etc.). It thereby allows the complexity of the learning procedure to increase linearly with the number of tasks. 
In particular, for a fixed assignment of parameters $(\delta_1, \ldots, \delta_{i-1}, \delta_{i+1}, \ldots, \delta_{2m})$, we estimate $\hat \delta_i$ with the value  $\overline \delta_i = \argmax_{\delta_i} \hspace{0.1cm}\mbox{$F$}(\bdelta)$, $i \in \{1, \ldots,  2m\}$. The learning procedure is sketched below:
\begin{enumerate}[1.]
\item  Randomly permute the vector $\bdelta$, and randomly initialize $\bdelta$;
\item  Determine an estimate $\overline \delta_i$ of $\hat \delta_i$ with a standard line search procedure for optimizing continuous functions of one variable, and fix $\delta_i = \overline \delta_i$;
\item  Iterate Step 2 for each $i \in \{1, 2, \ldots, 2m\}$;
\item  Repeat Step 3 till a stopping criterion is satisfied.
\end{enumerate}
 As stopping criterion we used a combination of the maximum number of iterations and of the maximum norm of the difference of two subsequent estimates $\overline \bdelta$  (falling below a given threshold).
As initial test, at Step 2 we simply adopted a grid search optimization algorithm, 
where a  set of trials is formed for each parameter, and all possible
parameter combinations are assembled and tested.
\subsubsection{Label inference.}\label{subsub:dyn}
 Once the parameters $\hat\bgamma, \hat\brho$ have been estimated, we consider the subnetwork $\mathcal{H_U}=\langle \bW_{\textrm{UU}}, \hat\bgamma, \hat\brho,  \bS\rangle$ restricted to the unlabeled nodes $U$, whose energy is
 \begin{equation}\label{eq:ExtEnergyS_U}
E_\mathcal{H_U}(\bU) = \sum_{k=1}^m \left( E_U\big(\bu^{(k)}\big) + \frac{\alpha}{2} \Big(S_k\sum_{i=1}^h u_{ik}^2 - \sum_{\substack{r=1\\r\neq k}}^{m} s_{kr}\sum_{i=1}^h u_{ik}u_{ir}\Big),
\right) 
\end{equation}
with $\bU=(\bu^{(1)}, \bu^{(2)}, \ldots, \bu^{(m)})$   state of $\mathcal{H_U}$, $\bu^{(k)} = (u_{1k},\ u_{2k},\ \ldots, u_{hk}) = (x_{1k},\ x_{2k},\ \ldots, x_{hk})\in \{\sin\hat\rho_k, -\cos\hat\rho_k \}^h$, $E_U\big(\bu^{(k)}\big) = -\frac{1}{2}{\bu^{(k)}}^T\bW_{UU}\bu^{(k)} + {\bu^{(k)}}^T\overline\btheta_k$, and $\overline\btheta_k = \hat\gamma_k\be_h - W_{\textrm{UL}}\bar\bl^{(k)}$ is the vector of activation thresholds for task $c_k$, including the contribution of labeled nodes (which are clamped).

In the case the learned parameters make $\bar\bL$ a part of global minimum of $\mathcal{H}$, by determining the global minimum of $\mathcal{H_U}$, we can successfully determine the global minimum of  $\mathcal{H}$ (as stated by the following theorem), and consequently the solution of the problem. 
 \begin{theorem}{} 
Given a \textit{multitask Hopfield network} $\mathcal{H}=\langle W, \bgamma, \brho, \bS\rangle$ on neurons $V$, bipartitioned into the sets $L$ and $U$, if $\bL$ is a part of a global minimum of the energy of $\mathcal{H}$, and  $\bU$ is a global minimum of the energy of $\mathcal{H}_U=\langle W_{UU}, \bgamma, \brho, \bS\rangle$, then $(\bL,\bU)$ is a global minimum of the energy of $\mathcal{H}$.
\label{th:subnet}\end{theorem}
On the other side, computing the energy global minimum of $\mathcal{H}_U$ would require time intensive algorithms; accordingly, to preserve the model efficiency and scalability, we run the dynamics of $\mathcal{H}_U$ till an equilibrium state is reached, which, in general, is an energy local minimum. 

Given an initial state $\bU(0)$, at each time $t$ one neuron is updated, and in $nm$ consecutive steps all neurons are updated asynchronously and in a randomly chosen order according to the following update rule:
\begin{equation}\label{eq:HN_U_update}
u_{ik}(t+1)=\begin{cases}
\sin \hat\rho, &\mbox{if } \phi^U_{ik}(t)>0\\
-\cos\hat\rho, &\mbox{if\ } \phi^U_{ik}(t)  \leq 0
\end{cases},
\end{equation}
where $u_{ik}(t+1)$ is the state of neuron $ik$ at time $t+1$, and $\phi^U_{ik}(t)$ is the restriction of $\phi_{ik}(t)$ to $U$.
According to Theorem~\ref{th:converg}, the dynamics (\ref{eq:HN_U_update}) converges to an equilibrium state $\bar\bU$ of $\mathcal{H}_U$, and the predicted bipartition  $(U_{k,+},U_{k,-})$  for task $k$ is: $U_{k,+} := \{i \in U | \bar u_{ik} = \sin \hat\rho\}$ and 
$U_{k,-} := \{i \in U | \bar u_{ik} = -\cos \hat\rho\}$.

\subsubsection{Dynamics regularization.}
As shown by~\cite{Frasca13}, the network dynamics might get stuck in trivial equilibrium states when input labeling are highly unbalanced ---e.g. states made up by almost all negative neurons.  
To prevent this behaviour, they applied a dynamics regularization, with the aim to control the number of positive neurons in the current state. By extending that approach, and denoted by $p_{k,+} = \frac{|L_{k,+}|}{|L|}$ the proportion of positives in the training set for task $c_k$, we add to the energy function $E_\mathcal{H_U}(\bU)$ the regularization term 
\begin{equation}
\eta_k\left(\sum_{i=1}^{h}(a_ku_{ik} + b_k) - hp_{k,+} \right)^2 ,
\label{eq:new_term}
\end{equation} 
where $a_k = \frac{1}{\sin\hat\rho_k + \cos\hat\rho_k}$, $b_k = \frac{\cos\hat\rho_k}{\sin\hat\rho_k + \cos\hat\rho_k} $, and $\eta_k$ is a real regularization parameter. 
Since $a_k$ and $b_k$ are such that $(a_ku_{ik} + b_k) = 1$ when $u_{ik}=\sin\hat\rho_k$, $0$ otherwise, 
the $\sum_{i=1}^{h}(a_ku_i + b_k)$ is the number of positive neurons in $\bu^{(k)}$. The term~(\ref{eq:new_term}) is thereby minimized when the number of positive neurons in $\bu^{(k)}$ is  $hp_{k,+}$. This choice is motivated by the fact that 
\[ hp_{k,+} = \arg \underset{q} \max \hspace{0.2cm}  Prob\left\{|U_{k,+}| = q \ \vert\ L \ \mbox{contains}\ |L_{k,+}|\  \mbox{positives} \right\}, 
\]
when $U$ and $L$ are randomly drawn from $V$ ---see~\cite{Frasca13}.  By simplifying eq.~(\ref{eq:new_term}),  up to a constant terms, we obtain the quadratic term:
\begin{displaymath}
\eta_k a_k\Big(a_k\sum\limits_{i=1}^h\sum\limits_{\substack{j=1\\j\neq i}}^h u_{ik}u_{jk} + \big(2b_k(h-1)+1-2p_{k,+}  \big)\sum\limits_{i=1}^h u_{ik}\Big),
\end{displaymath}
which can be thereby included into $E_U(\bu^{(k)})$:
\begin{displaymath}
E_U\big(\bu^{(k)}\big) = -\frac{1}{2} \sum_{i=1}^{h}\sum_{\substack{j=1 \\ j \neq i}}^{h} w_{ij}^{(k)}u_{ik} u_{jk} + \sum_{i=1}^{h} u_{ik}\tilde{\theta}_{ik},
\end{displaymath}
where $\tilde{\theta}_{ik} = \overline\theta_{ik} + \eta_k a_k\left[2b_k(h-1)+(1-2p_{k,+}h)\right]$ and $w_{ij}^{(k)}=(w_{ij} - 2\eta_k a_k^2)$. By adding a regularization term for each task $c_k$, we obtain the following overall energy:
\begin{equation*}
\begin{split}
    E_\mathcal{H_U}(\bU)=\sum_{k=1}^m \Bigg(-\frac{1}{2}{\bu^{(k)}}^T\bW^{(k)}_{\textrm{UU}}\bu^{(k)} + &{\bu^{(k)}}^T\tilde\btheta_k + \\& \frac{\alpha}{2} \Big(S_k\sum\limits_{i=1}^h u_{ik}^2 - \sum\limits_{\substack{r=1\\r\neq k}}^{m} s_{kr}\sum_{i=1}^h u_{ik}u_{ir}\Big)\Bigg)
\end{split}
\end{equation*}
Informally, this regularization leads to a different network topology for each task, in addition to a modification of the neuron activation thresholds. Nevertheless, since the connection weights are modified by a constant value, from an implementation standpoint this regularization just need to memorize $m$ different constant values, thus not increasing the space complexity of the model. 
As preliminary approach, and to have a fair comparison, the parameters $\eta_k$ have been set as for the singletask case \cite{Frasca13},
that is $ \eta_k = \begin{array}{rl}
\beta \big|\tan\big((\hat\rho_k - \frac{\pi}{4})*2\big)\big| \hspace{.0cm}& \mbox{,}  
\end{array}$
where $\beta$ is a non negative real constant. Another advantage of this choice is that we have to learn just one parameter $\beta$, instead of $m$ dedicated parameters.
\section{Preliminary results and discussion}
In this section we evaluate our algorithm on the prediction of the bio-molecular functions of proteins, a binary classification problem aiming at associating sequenced proteins with their biological functions. In the following, we describe the experimental setting, analyze the impact on performance of parameter configurations, and finally we compare our algorithm against  other state-of-the-art graph-based methods.
\subsection{Benchmark data}
In our experiments we considered the Gene Ontology~\cite{GO00} terms, i.e. the reference functional classes in this context,  and their annotations to the \textit{Saccaromyces cerevisiae} (yeast) proteins, one of the most studied model organisms. The connection matrix $\bW$ has been  retrieved from the STRING database, version 10.5~\cite{STRING10}, and contains $6391$ yeast proteins.
As common in this context, the GO terms with less than $10$ and more than $100$ yeast protein annotations (positives) have been discarded, in order to have a minimum of information and to avoid too generic terms ---GO is a DAG, where annotations for a term are transferred to all its ancestors.  
We considered the UniProt GOA  (release $87$, 12 March 2018) experimentally validated annotations from all GO branches, Cellular Component (CC), Molecular Function (MF) and Biological Process (BP), for a total of $162$, $227$, and $660$ CC, MF, BP GO terms, respectively. 
\subsection{Evaluation setting}
To evaluate the generalization capabilities of our algorithm, we used a $3$-fold cross validation (CV), and measured
the performance in terms of  Area Under the ROC curve (AUC) and Area Under the Precision-Recall curve (AUPR). The AUPR has been adopted in the recent CAFA2 international challenge for the critical assessment of protein functions~\cite{CAFA2}, since in this imbalanced setting AUPR is more informative than AUC~\cite{Saito15}. 
\subsection{Model configuration}
HoMTask has three hyper-parameters, $\tau$, $\beta$ and $\alpha$, and two functions to be chosen: $f$ in eq.~(\ref{eq:fuzzy_memb}), and $\sigma$ in eq.~(\ref{eq:F}).
$\tau$, $\beta$  and $\alpha$ were learned through internal $3$-fold CV, considering also the cases $\alpha$ and $\beta$ in turn or together clamped to $=0$, to evaluate their individual impact on the performance. A different discussion can be made for the $\tau$ parameter, since in our experimentations best performance correspond to large values of $\tau$ (e.g. $\tau>500$), thus making the model less sensitive to this choice (the function $f$ becomes a Heaviside function). This behaviour apparently conflicts with results reported in~\cite{homcat}, where typically  $0.5<\tau<2$ performed best. 
However, in that work the authors focused on a substantially different learning task, i.e.  a singletask Hopfield model, where nodes were divided into categories, and the model parameters were not related to different tasks, but to different node categories. We still include $\tau$ in the formalization proposed in Section~\ref{sub_homtask} because it permits also future analytic studies about the derivatives of $\sigma$, to determine close formulations for the optimal parameters. Further,
We set $f(x)=\frac{1}{2}\big(\frac{2}{\pi} \arctg(x)+1\big)$, since this choice in a multi-category context leaded to excellent results~\cite{homcat}, even if different choices are possible (Section~\ref{sub_homtask}).

On the other side, we tested two choices for $\sigma$: the harmonic mean ($\sigma_1$) and mean functions ($\sigma_2$).
Furthermore, another central factor of our model is the computation of the task similarity matrix $\bS$, which can be computed by using several metrics (see for instance~\cite{gene2disco}), and how to group the tasks that should be learned together. We employed in this work the Jaccard similarity measure, since it performed nicely in hierarchical contexts~\cite{gene2disco,VASCON18,MTLP}, defined as follows:
\begin{displaymath}
s_{kr} =  \left\{ \begin{array}{cl}
{\displaystyle \frac{\big|L_{k,+} \wedge L_{r,+}\big|}{\big|L_{k,+} \vee L_{r,+}\big|} } & \text{if $L_{k,+} \vee L_{r,+} \neq \emptyset$} 
\\
0  & \text{otherwise.}
\end{array}
\right.
\end{displaymath}
Thus, $s_{kr}$ is the ratio between the number of instances that are positive for both tasks and the number of instances that are positive for at least one task. The higher the number of shared instances, the higher the similarity (up to $1$); conversely, if two task do not share items, their similarity is zero. 

Finally, we grouped tasks by GO branch, and by GO branch and number of positives: in the first case (\textit{Branch}), all tasks within a GO branch are learned simultaneously; in the latter one (\textit{Card}), tasks in the same branch having $10$-$20$, $21$-$50$, or $51$-$100$ positives have been grouped together. Both approaches are quite usual when predicting GO terms~\cite{Mostafavi10,MTLP}.   

Table~\ref{tab:config} reports the obtained results on the $CC$ terms.
First, the two different strategies for grouping tasks led to similar results in this setting, with the \textit{Branch} grouping  being experimentally slower because the learning procedure needs more iterations to converge when the number of parameters increases (due to the max norm adopted here as stopping criterion). Nevertheless, we remark that no thresholding on the matrix $\bS$ has been applied in both cases; thus, in the same model even tasks with small similarities can be included, which in principle might introduce noise in the learning and inference processes. Consequently, the advantage of jointly learning a larger number of similar tasks can be compensated by this potential noise;  investigating other task grouping and similarity thresholding strategies could thereby give rise to further insights about model, which for lack of room we destine to future study. 

\begin{table}[t]
\caption{Performance averaged across CC terms for different configuration of the model.}\label{tab:config}
\centering
\small
\vspace{-0.2cm}
\begin{tabular}{lcc}
  \hline
Configuration & AUC & AUPR \\ 
Branch, $\sigma_1$ & 0.961 & 0.439\\
Card, $\sigma_1$ & 0.959 & 0.439\\
Card, $\sigma_1$, $\alpha=0$ & 0.959 & 0.431\\
Card, $\sigma_1$, $\beta=0$ & 0.810 & 0.204\\ 
Card, $\sigma_1$, $\alpha=\beta=0$ & 0.811 & 0.204\\ 
Card, $\sigma_2$ & 0.937 & 0.312 \\ 
   \hline
\end{tabular}
\vspace{-0.35cm}
\normalsize
\end{table}

Regarding the impact of parameter $\beta$, regulating the effect of dynamics regularization, a strong decay in performance is obtained when no regularization is applied ($\beta = 0$): this confirms the tendency of the network trajectory to be attracted in some limit cases by trivial fixed points, already observed in the singletask Hopfield model~\cite{Frasca13}. In this experiment, the contribution of regularization is still more dominating, since it allows to double the AUPR performance.

Less impact on the performance apparently has the parameter $\alpha$, regulating the multitask energy term (Eq. (\ref{eq:ExtEnergy})). Indeed, the performance reduces just of $0.008$ when $\alpha=0$; however, this behaviour should be further studied, because it can be strictly related to the noise we introduced by grouping tasks without filtering out connections between less similar task. Thus, further experiments with different organisms would help this analysis and potentially reveal novel and more clear trends. 
It is also important noting that by setting $\alpha=0$, the overall multitask contribution is not cancelled: the learning procedure, by maximizing criterion~(\ref{eq:F}), still learns tasks jointly, even when the multitask contribution in formula~(\ref{eq:fuzzy_memb}) is removed. For instance, choosing $\sigma$ equal to the minimum function would mean learning individual task parameters in order to maximize the minimum performance ($\min_k F_k$) across tasks, even when $\alpha=0$. 

Finally, the function $\sigma$ itself seems having a marked impact on the model. When using the mean function ($\sigma_2$) the AUPR decreases of around $25\%$ with respect to the AUPR obtained using the harmonic mean ($\sigma_1$). To some extent such a result is expected, since the harmonic mean tends to penalize more the outliers towards $0$, thus fostering the learning procedure to estimate the parameters in order not to penalize some tasks in favors of the remaining ones, which instead can happen when using the mean function.

This preliminary analysis about the model suggested to adopt the configuration ``Card, $\sigma_1$" in the comparison with the state-of-the-art methodologies, which is described in the next section. 
\subsection{Model performance}\label{sub:soaComp}
We compared our method with several state-of-the-art graph-based algorithms, ranging from  singletask Hopfield networks and other multitask methodologies, to some methods specifically designed to predicting protein functions: \textit{RW, random walk}~\cite{Lovasz96}, the classical $t$-step random walk algorithm, predicting a score corresponding to the probability that a $t$-step random walk in $G$, starting from positive nodes, ends in the node to be predicted; \textit{RWR, random walk with restart}, since in RW after many steps the walker may forget the prior information coded in the initial probability vector (0 for negative nodes  $1/|L_{k,+}|$ for positive nodes), RWR allows the walker to move another random walk step with probability $1-\theta$, or to restart from its initial condition with probability $\theta$; \textit{GBA, guilt-by-association}~\cite{Schwikowski00}, a method  based on the assumption that interacting proteins are more likely to share similar functions; \textit{LP, label propagation}~\cite{Zhu03}, a popular semi-supervised learning algorithm which propagates labels to unlabeled nodes through an iterative process based on Gaussian random fields over a continuous state space; \textit{MTLP, MTLP-inv}~\cite{MTLP}, two recent multitask extensions of LP, exploiting task dissimilarities (MTLP) and similarities (MTLP-inv); \textit{MS-kNN, Multi-Source k-Nearest Neighbors}~\cite{Lan13}, a method based on the $k$-Nearest Neighbours ($k$NN) algorithm~\cite{Altman92}, among the top-ranked methods in the recent CAFA2 international challenge for AFP~\cite{CAFA2}; \emph{RANKS}~\cite{RANKS}, a recent graph-based method proposed to rank proteins, adopting a suitable kernel matrix to extend the notion of node similarity also to non neighboring nodes; \textit{COSNet}. We used the same approach in~\cite{Frasca13bis} to compute a node ranking, necessary to calculating both AUC and AUPR.

In Table~\ref{tab:SOA_COMP} we show the obtained results. Our method achieves the highest AUPR in all the experiments, with statistically significant difference over the second top method (RWR) in $2$ out of $3$ experiments (Wilcoxon signed rank test, $\alpha=0.05$). The  performance improvement  compared with COSNet is noticeable, showing the remarkable contribution supplied by our multitask extension. Interestingly, MTLP and MTLP-inv increase the AUPR results of LP not so remarkably as HoMTask: this means that the further information regarding task similarities should be appropriately exploited in order to achieve relevant gains. RANKS is the third method in all experiments, followed by  MTLP(-inv), while MS-$k$NN is surprisingly the last method.  
Our method achieves good results also in terms of AUC (which however is less informative in this context), being close to top performing methods (RWR on CC and MF, and MTLP-inv on BP terms). 
\begin{table}[t]\label{tab:SOA_COMP}
\caption{Performance comparison averaged across GO branches. In bold the top results, underlined when statistically different than the second top result.}
\centering
\scriptsize
\begin{tabular}{lcccccccccc}
  \hline
 & RW & RWR & GBA & LP & MTLP & MTLP-inv & MS-kNN & RANKS & COSNet & HoMTask\\ 
  \hline
 \multicolumn{11}{c}{AUC}\\
CC & 0.954 & \textbf{0.966} & 0.944 & 0.964 & 0.957 & 0.964 & 0.790 & 0.958 & 0.904  & 0.959\\ 
  MF & 0.934 & \textbf{0.955} & 0.931 & 0.951 & 0.939 & 0.953 & 0.742 & 0.945 & 0.859 & 0.945\\ 
   BP & 0.943 & 0.959 & 0.935 & 0.955 & 0.947 & \textbf{0.961} & 0.764 & 0.949 & 0.855 & 0.954\\  
   \hline
\multicolumn{11}{c}{AUPR}\\
CC & 0.367 & 0.437 & 0.207 & 0.308 & 0.343 & 0.342 & 0.218 & 0.398 & 0.361 & \textbf{0.439}\\ 
  MF & 0.199 & 0.272 & 0.125 & 0.201 & 0.229 & 0.234 & 0.090 & 0.236 & 0.214 & \underline{\textbf{0.291}}\\ 
   BP & 0.244 & 0.313 & 0.145 & 0.224 & 0.246 & 0.250 & 0.116 & 0.271 & 0.241 & \underline{\textbf{0.330}}\\  
   \hline\\
\end{tabular}
\vspace{-0.8cm}
\normalsize
\end{table}
\section*{Conclusions}
We have proposed the first multitask Hopfield Network for classification purposes, HoMTask, capable to simultaneously learn multiple tasks and to cope with the label imbalance. In our validation experiments, it significantly outperformed singletask HNs, and favorably compared with state-of-the-art single and multitask graph-based methodologies. 
Future investigations might reveal novel insights about the model, in particular regarding the choice of the task relatedness matrix, the task grouping strategy, the multitask criterion to be optimized during the learning phase, and the optimization procedure itself.     
\bibliographystyle{splncs04}
\bibliography{HMTred}

\end{document}